\documentclass[
reprint,
superscriptaddress,
amsmath,
amssymb,
pra,
floatfix,
twocolumn,
longbibliography
]{revtex4-1}

\usepackage{float}
\usepackage{amsfonts,amsmath,amssymb,mathtools}
\usepackage{braket}
 \usepackage{array,bm,color}
\usepackage{epsfig,graphicx,nomencl,revsymb4-1,upgreek,url}
\usepackage{hyperref}
\usepackage{caption}
\usepackage{subcaption}
\usepackage{graphicx}
\usepackage{calc}
\usepackage{siunitx}
\newcolumntype{Y}{>{\centering\arraybackslash}X}
\graphicspath{{./figures/}}

\newcommand{\poly}{\mathrm{poly}}
\usepackage{amssymb}

\newcommand{\note}[1]{}

\usepackage{amsmath}
\usepackage{amssymb, cancel}
\usepackage{physics} 
\usepackage{mathtools}
\usepackage{hyperref}
\usepackage{exscale}
\usepackage{graphicx,nicefrac} 
\usepackage{color}
\usepackage{stackengine}
\usepackage{txfonts}
\usepackage{amsmath}
\usepackage{algorithm}
\usepackage{mathtools}
\usepackage{hyperref}
\hypersetup{colorlinks=true, citecolor=magenta, linkcolor=magenta, urlcolor=magenta}
\usepackage{footnote}
\usepackage{graphicx}
\usepackage{subcaption}
\usepackage{color}
\usepackage{blkarray}
\usepackage{dsfont}
\usepackage{mathrsfs}
\usepackage{bbold}
\usepackage{bbm}

\newcommand{\Qxit}{\tilde{\textbf{Q}}_{\xi}}

\newcommand{\II}{\textbf{I}}
\newcommand{\XX}{\boldsymbol{X}}

\newcommand{\EX}{\langle \boldsymbol{X} \rangle}

\newcommand{\IPC}{{\rm IPC}}

\newcommand{\DD}{\boldsymbol{D}}
\newcommand{\VV}{\boldsymbol{V}}

\usepackage{tikz}

\graphicspath{{./figures/}}

\usepackage{amsthm}
\usepackage{tikz}
\usepackage{amsmath}
\usepackage[format=hang, 
justification=raggedright]{caption}
\usepackage{cleveref}
\def\vec{\mathaccent "017E\relax }
\usepackage{microtype}
\usepackage{graphicx}
\usepackage{booktabs} 
\usepackage{hyperref}
\usepackage{amsmath}
\usepackage{amssymb}
\usepackage{mathtools}

\theoremstyle{plain}
\newtheorem{theorem}{Theorem}[section]

\newtheorem{lemma}[theorem]{Lemma}
\newtheorem{corollary}[theorem]{Corollary}
\theoremstyle{definition}
\newtheorem{definition}[theorem]{Definition}

\theoremstyle{remark}

\usepackage[textsize=tiny]{todonotes}

\begin{document}

\author{Anthony M. Polloreno}
\affiliation{JILA, NIST and Department of Physics, University of Colorado, 440 UCB, Boulder, Colorado 80309, USA}
\email[Email: ]{ampolloreno@gmail.com}
\date{\today}
\title{Restrictions on Physical Stochastic Reservoir Computers}

\begin{abstract}
Reservoir computation is a recurrent framework for learning and predicting time series data, that benefits from extremely simple training and interpretability, often as the the dynamics of a physical system. In this paper, we will study the impact of noise on the learning capabilities of analog reservoir computers. Recent work on reservoir computation has shown that the information processing capacity (IPC) is a useful metric for quantifying the degradation of the performance due to noise. We further this analysis and demonstrate that this degradation of the IPC limits the possible features that can be meaningfully constructed in an analog reservoir computing setting. We borrow a result from quantum complexity theory that relates the circuit model of computation to a continuous time model, and demonstrate an exponential reduction in the accessible volume of reservoir configurations. We conclude by relating this degradation in the IPC to the fat-shattering dimension of a family of functions describing the reservoir dynamics, which allows us to express our result in terms of a classification task. We conclude that any physical, analog reservoir computer that is exposed to noise can only be used to perform a polynomial amount of learning, despite the exponentially large latent space, even with an exponential amount of post-processing.
\end{abstract}

\maketitle

\section{Introduction}
Reservoir computers \cite{2021, Tanaka_2019} are a particular kind of recurrent neural network where the only trained parameters are outgoing weights forming a linear layer between the internal parameters of the network, called the reservoir, and the readouts. This architecture drastically simplifies the process of training the network while maintaining high computational power. A defining aspect of a reservoir computer is its ability to perform inherently temporal tasks, such as time-series prediction or pattern recognition within sequences \cite{dominey1995complex}. In this framework, the reservoir serves as a temporal kernel \cite{dong2020reservoir}, transforming the input sequence into a high-dimensional state represented in the hidden state of the reservoir. These hidden states are then linearly combined by trainable output nodes. The high-dimensional, temporal representation of data in the reservoir gives the output nodes enough flexibility to extract complex, non-linear features while avoiding the common issue of vanishing/exploding gradients or overfitting found in other recurrent architectures. 

Notably, due to the simple linear nature of the learned part of the network, the behavior and outcomes of reservoir computers are defined primarily by the reservoir. This interpretability sets reservoir computers apart from other learning approaches, making them valuable for applications requiring transparency or insight into the learning process. A crucial advantage is that the non-linearities needed for learning are encapsulated within the physical dynamics of the reservoir, suggesting that we can leverage physical understanding to glean insights into the learning. In this paper, we explore the impact of noise on the performance of reservoir computers.

We will assume that the stochasticity of the reservoir comes from e.g., dissipation, incomplete information or measurement shot noise. In addition, we will assume the length of time the reservoir is run and the amount of energy the reservoir consumes are both polynomial in the system size $n$. This constrains the kinds of allowed operations - for instance, it is not possible to perform $n$-body operations (arbitrary long range interactions) in a single time step, since these require an exponential number of $k$-body ($k < n$) operations in general. In our arguments we assume that the dynamical system corresponds to a digital computer, and hence $k-$body operations correspond simply to $k-$bit logic gates.

Our first result, \Cref{theorem:generateddesign}, will demonstrate that the dimension of the space we can perform regression onto can be meaningfully identified as polynomial, despite the exponential number of functions being considered. In particular, we will show that a large fraction of the signals produced by a reservoir are strongly dominated by noise, and that it requires a lot of waiting to distinguish one such signal from another. Our second result, \Cref{theorem:nopoly} will give a characterization of the class of functions represented by the reservoir. Namely, due to a theorem of Bartlett et al.\cite{bartlett1994fat}, we will see that is possible to relate a certain kind of learnability of this class to its fat-shattering dimension \cite{bartlett1994fat}. 

\section{Background}

Prior work \cite{boedecker2012information,baranvcok2014memory} on the information processing and memory capacity of reservoir computers often considers information theoretic measurements to assess their performance, or considers quantities such as spectral radii \cite{alexandre2009reservoir}, Lyapunov exponents \cite{chen2022time} and their connection to other dynamical systems and Koopman operators \cite{gulina2021two}. A typical family of networks considered in reservoir computing are so-called echo state networks \cite{jaeger2001echo}, which are deterministic recurrent reservoirs. By considering deterministic networks that are implemented with physical reservoirs that may be subject to noise in the physical degrees of freedom, such as uncertainty due to fluctuations in measurement voltages or thermal fluctuations, we are naturally led to consider stochastic reservoir computers. 

Recent theoretical analyses \cite{polloreno2023note, hu2023tackling} have studied in detail the impact of such noise on the IPC and in this work we also choose to use the IPC due to its simple linear algebraic interpretation. In particular, reservoir computing is an efficient approach for learning with dynamical systems, and this work aims to provide insights into this class of methods. While the IPC is not a standard measure for information processing capabilities of artificial neural networks, it provides an effective way of characterizing the maximum information processing ability of a network and, as we will see, can also be used to measure the degradation of performance due to stochasticity. The IPC has the additional benefit of being interpretable as the dimension of the hypothesis space associated with reservoir computer.

In our discussion we rely heavily on asymptotic notation, and so we remind the reader here of a few definitions. \begin{enumerate}
\item $f(n)= o(g(n))$ (Little-o): This is used when $f(n)$ grows strictly slower than $g(n)$ as $n$ approaches infinity. Formally, we say $f(n)$ is $o(g(n))$ if for any positive constant $k > 0$, there exists a value $n_0$ such that for all $n > n_0$, it holds that $f(n) < k \cdot g(n)$. In other words, no matter how small a constant we choose, $f(n)$ will eventually be surpassed by $k \cdot g(n)$.

\item $f(n)= O(g(n))$ (Big-O): This is used when $f(n)$ grows at the same rate or slower than $g(n)$ as $n$ approaches infinity. Formally, $f(n)$ is $O(g(n))$ if there exist constants $k$ and $n_0$ such that for all $n > n_0$, it holds that $f(n) \leq k \cdot g(n)$. In other words, we can find an upper bound on the growth rate of $f(n)$.

\item $f(n)= \omega(g(n))$ (Little-$\omega$): This is used when $f(n)$ grows strictly faster than $g(n)$ as $n$ approaches infinity. Formally, $f(n)$ is $\omega(g(n))$ if for any positive constant $k > 0$, there exists a value $n_0$ such that for all $n > n_0$, it holds that $f(n) > k \cdot g(n)$. This is like being strictly greater ``in the limit''.

\item $f(n)= \Omega(g(n))$ (Big-$\Omega$): This is used when $f(n)$ grows at the same rate or faster than $g(n)$ as $n$ approaches infinity. Formally, $f(n)$ is $\Omega(g(n))$ if there exist constants $k$ and $n_0$ such that for all $n > n_0$, it holds that $f(n) \geq k \cdot g(n)$. This is saying $f(n)$ has a growth rate that is at least the rate of $g(n)$.
\end{enumerate} 

Put simply, $o$ and $O$ notations provide upper limits (strict and non-strict respectively) on the growth of a function, while $\omega$ and $\Omega$ provide lower limits (strict and non-strict respectively) on the growth of a function. They are effectively bounding the function's growth from below or above. It is worth noting that $f(n) = o(g(n))$ is equivalent to $g(n) = \omega(f(n))$ and $f(n) = O(g(n))$ is equivalent to $g(n) = \Omega(f(n))$.

\section{Reservoir Computers}
A reservoir computer is  a dynamical system, generally described by a system of differential equations, driven by an input $\boldsymbol{U}(t) \in \mathbb{R}^m$ and described by time varying degrees of freedom $\XX(t) \in \mathbb{R}^n$ which represent its state. In this paper, we will consider systems consisting of $n$ degrees of freedom, which we will assume are bits. We note that this assumption is not particularly restrictive, and is in fact commonplace - any kind of physical dynamics can be encoded with vanishing error into a discretized signal. Because these degrees of freedom can be multiplied together to form new outputs, $\boldsymbol{Y}(t) \in \mathbb{R}^{d}$, in general the output dimension is $d\geq n$ and in this work we will consider $d=2^n$ corresponding to all possible products of the $n$ signals.

In the case of a stochastic reservoir computer, the state of the computer is in general a vector in the $2^n$ dimensional space of probability distributions over bitstrings of length $n$, i.e.
$\boldsymbol{Y}(t) = (p_{0...0}(t), ..., p_{1...1}(t))$. While it may be compelling to assume that the $n$ single-bit marginal distributions will contain the most computational utility, we note that arbitrary reservoirs can be used to construct complex and potentially highly correlated probability distributions. Thus, any argument that suggests there is a preferred collection of bitstrings among the $2^n$ possible ones is making further assumptions about the structure of the reservoir.

Typically, in a reservoir computing setting, one considers a dynamical system observed at discrete time-steps $t = 0, 1, 2, \ldots$, and the outputs are used to approximate a target function. Due to the general presence of memory in dynamical systems, we additionally define the concatenated $h$-step sequence of recent inputs $\boldsymbol{U}^{-h}(t) = [\boldsymbol{U}(t - h + 1), \boldsymbol{U}(t - h + 2), \ldots, \boldsymbol{U}(t)]$. While we may use the reservoir to learn a function of time, the reservoir's degrees of freedom themselves can be approximated by maps $x_k^h:\boldsymbol{U}^{-h}(t)\mapsto\mathbb{R}$. In particular, this is because we require that the reservoir satisfies the fading memory property. A dynamical system has fading memory if, for all $\epsilon > 0$, there exists a positive integer $h_0\in\mathbb{N}$, such that for all $h>h_0$, for all initial conditions, and for all sufficiently long initialization times $T'>h$, the $x_k(t)$ at any time $t \geq 0$ are well-approximated by functions $x_k^h$:
\begin{equation}\label{eq:fadingmem}
\mathbb{E}((x_k(t) - x_k^h[\boldsymbol{U}^{-h}(t)])^2) < \epsilon
\end{equation}
where the expectation is taken over the $t + T'$ previous inputs. 
Due to two different sources of randomness in this paper, we next give definitions and notation for the different averages we compute, before we define the capacity of reservoir to reconstruct a signal.
\begin{definition}\label{def:averages}
The input signal to a reservoir computer corresponds to a stochastic random variable and thus has an associated probability measure $\mu$, and we write $\overline{f}$ to denote averages of a quantity $f$ with respect to this measure, i.e.
\begin{equation}
    \overline{f} = \int d\mu f.
\end{equation}
A stochastic reservoir computer \cite{ehlers2025stochastic} further has probabilistic dynamics, coming from, e.g., noise. Thus the outputs have an additional associated probability measure $\nu$, and we write $\langle f \rangle$ to denote the average of a quantity $f$ over this measure, i.e.
\begin{equation}\label{eq:stochmeas}
    \langle f \rangle = \int d\nu f.
\end{equation}
\end{definition}
\begin{definition}\label{def:capacity}
\cite{dambre2012information} The capacity of a reservoir to reconstruct a signal $y(t)$ is given as
\begin{align}
    C_T[y] = 1 -  {\min}_{\omega}\frac{ \sum_{t=1}^T(\hat{y}_{\omega}(t) - y(t))^2  }{ \sum_{t=1}^T y^2(t) },
\end{align}
where $\hat{y}_{\omega}$ is the estimate of $y$ produced from the reservoir . 
\end{definition}
Typically $\hat{y}_{\omega}$ is produced via a linear weighting of the output signals, i.e. $\hat{y}(t) = \boldsymbol{\omega}^T\boldsymbol{Y}(t)$ for a weight vector $\boldsymbol{\omega}$. Then, we can define the IPC as
\begin{definition}\label{def:ipc}
\cite{dambre2012information} For a complete and countably infinite set of orthogonal basis functions $\{y_1, y_2, \ldots \}$, the IPC of a dynamical system can be defined as
\begin{align}\label{eq:ipc}
    {\rm IPC} = \lim_{D \to \infty} \lim_{T \to \infty} \sum_{\ell}^{D} C_T[y_{\ell}] \leq n.
\end{align}
\end{definition}

References \cite{hu2023tackling} and \cite{polloreno2023note} derive a closed form expression for the IPC of a stochastic reservoir, which we state here without derivation as \Cref{theorem:ipcdef}. Writing $\EX$ as short hand for $\boldsymbol{X}(t)$, we can perform a spectral decomposition $\overline{\EX \EX^T} = \VV \DD \VV^T$ with positive definite, diagonal matrix $\DD$ and an orthogonal matrix $\VV$ to define the generalized noise-to-signal matrix $\Qxit$ as
\begin{align}
     \II + \Qxit  \coloneqq \DD^{-\frac12} \VV^T\overline{\langle \XX \XX^T\rangle}\VV \DD^{-\frac12},
\end{align}
where $\II$ is the identity matrix, so that $\Qxit$ describes the deviation of the reservoir from an ideal noiseless reservoir that produces orthogonal outputs. 

\begin{definition}{\cite{hu2023tackling}}
The right eigenvectors of $\Qxit$ are called the eigentasks of the reservoir.
\end{definition}

In \cite{polloreno2023note} it is shown that $\Qxit$ gives some measure of the reservoir stochasticity in the basis of the correlations imposed by the input signal and gets its utility from the following theorem:
\begin{theorem}{\rm\cite{hu2023tackling, polloreno2023note}}\label{theorem:ipcdef}
The IPC of a stochastic reservoir is given as

    \begin{equation}
    \begin{split}
    {\rm IPC}  
    & =  
    \Tr( (\II + \Qxit)^{-1}) \\
    & = 
    \sum_{k=1}^n \frac{1}{1 + \tilde{\sigma}_k^2} \le n,
        \end{split}
    \end{equation}

\end{theorem}
\noindent where $\tilde{\sigma}_k^2$ are the eigenvalues of $\Qxit$. These eigenvalues correspond to noise-to-signal ratios, the inverse of SNRs, of the reservoir at performing their respective eigentasks. The reader may notice that this takes a similar form to the least squares solution to the linear regression problem with uncertainties on both the independent and dependent variables, and can be shown to come from similar considerations \cite{hu2023tackling}. The outputs of the reservoir are, in general, post-processed depending on the learning task at hand. Because we conventionally optimize over linear weights, we are free to define the outputs of the reservoir up to a linear transformation without impacting the IPC. In particular, we will find particularly convenient the probability representation of the reservoir outputs.
\begin{definition}\label{def:probrep}
The probability representation of the outputs of a reservoir is given by the bitstring probabilities $p_k$, i.e. the output signal is given by $X(t) = (p_{0...0}(t), \ldots, p_{1...1}(t))$.
\end{definition}

We note that, in our definition, we explicitly use all possible products (monomials) of the reservoir signals to define the output vector $Y(t)$, rather than directly using the raw signals themselves. This choice is motivated by three complementary considerations. First, the set of all possible products of binary signals forms a natural and complete basis for functions defined on binary spaces, making this construction standard in classical Boolean circuit theory. Second, in quantum computing, physical readouts and observables are inherently represented as monomials of measurement operators, such as products of Pauli operators. This directly aligns our theoretical framework with practical quantum computing implementations, where these monomials naturally represent computational observables. Third, by explicitly including these nonlinear product outputs, we establish a comprehensive theoretical approach that clearly delineates the intrinsic computational capabilities of the reservoir itself from nonlinear post-processing operations performed externally. While simpler or alternative nonlinear transformations, such as the Hilbert transform used in spintronic oscillator reservoirs \cite{tsunegi2023information}, are often practical, our general and explicit choice facilitates rigorous theoretical analysis of computational constraints and capabilities.

Our results in this paper will be similar in spirit to the results in \cite{poulin2011quantum, Shannon1949}, which show that the space of states accessible by a physical computer, defined next, are exponentially vanishing in the total state space. Conceptually, these results suggest our ultimate result - how can physical states give rise to signals that have useful support on all $2^n$ basis vectors if the states themselves are exponentially vanishing?
\begin{definition}\label{def:physical}
    We define a stochastic reservoir as physical if, motivated by \cite{poulin2011quantum}, its dynamics can be simulated by quantum or classical circuits in a time that scales polynomially with the system size $n$. More precisely, let a $k$-body circuit element be defined as a circuit component acting on at most $k$ inputs. We require that the dynamics of such a reservoir be implementable by circuits composed of $k$-body elements, with $k > 0$ being a fixed integer independent of $n$.
\end{definition}
The term physical emphasizes computational tractability and physical realizability, indicating that such reservoirs do not exhibit complexity scaling exponentially with system size. The stochastic aspect of the reservoir may arise either from observational (measurement) noise, dynamical noise inherent in the reservoir's evolution, or both, depending on the specific physical implementation considered. Specifically, observation noise can be trivially included as dynamical noise that occurs at the observation frequency, and dynamical noise can be associated with a measure as in \Cref{eq:stochmeas}. Because for our purpose, it suffices to say that there exists some measure associated with some stochastic process, we leave the detailed study of this to future work.

Intuitively, \Cref{def:physical} rules out states that are not practically accessible to the reservoir. The $k$-body requirement stems from physicality constraints on the density of circuit elements. There are only so many circuit elements that can fit into a space and increasing this density arbitrarily with $n$, i.e. $n-$bit operations, is physically impossible. It is of course possible to compile arbitrary $n-$body terms into $k-$body terms, but this in general will require an exponential number of gates, and thus exponential time. For the purposes of this paper we borrow the definition from \cite{poulin2011quantum} and provide a more detailed justification in \Cref{app:physical}. In particular, we additionally prove a lemma that we will make later use of.
\begin{lemma}\label{lemma:prob_changes}
Consider the magnitude of the input to an $n-$bit reservoir computer, given as $u(t) = ||\boldsymbol{U}(t)||_2$. Then, the changes in probabilities $\frac{dp(u)}{du}$ are $O(\poly(n))$. 
\end{lemma}
\begin{proof}
    In \Cref{app:physical} we argue from typical physical arguments that the changes in the probabilities $p(u)$ can only be changed polynomially rapidly in $u$. Furthermore, we argue in \Cref{app:physical} that by construction of our definition of \textit{physical}, we have only considered systems where $u$ is a polynomial in the system size $n$.
\end{proof}

Finally, we comment on the memory of the system. In our definition of a reservoir computer, we have required that the system have finite memory. Noise can serve as one mechanism to limit the memory of the system, since heuristically a system with an effective error rate of $p$ will experience an error every $1/p$ time steps. In the next section we will consider an integral over the effective input space, defined implicitly by the measure $\mu$ and which depends in turn on the memory of the system. In this work, we assume that the memory of the system is $O(\poly(n))$. This rules out, for example, error-corrected digital reservoirs that can simply remember a super polynomial number of inputs and use these to configure a super polynomial number of output signals.        

\section{Stochastic Reservoirs have Subexponential Capacity}\label{sec:subexp_cap}
Generally speaking, the computational utility of a reservoir computer is fully characterized by the dimension of its externally observable dynamics. In the case of a deterministic reservoir, the state space is fully specified by $n$ bits. It is possible to further construct all $2^n$ functions on these bits, which then give the potential for $2^n$ capacity arising from correlations between the bits. (These new signals can of course fail to give additional IPC, for example consider a reservoir with outputs $f_1, f_2$ and $f_1f_2$, with $\int d\mu(u) f_1f_2 = 0$, where we have considered the standard $L^2$ inner product.) For example, the collection of polynomials $S$, given by $S = \{x, x^2, x^4, x^8...x^{2^n}\}$, where all $2^n$ elements of the powerset $2^S = \{\{\},\{x\}, \{x^2\},..., \{x,x^2\}...\}\}$ can be used to construct a collection $S'$ of exponentially many linearly independent polynomials by through multiplication, i.e. $S' = \{x, x^2, x^3..., x^{2^n}\}$.

As previously discussed, in the case of a stochastic reservoir, the state space is immediately naturally defined as $d=2^n$ dimensional. However, despite the system requiring $2^n$ real numbers to be described, a natural question is if it is possible to utilize this $2^n$ dimensional space for useful computation, and in our case, learning? In particular we are able to construct $2^n$ signals by taking multiplicative products of the $n$ output signals  - does this provide an exponential amount of IPC as in the deterministic case? We will find the answer is no, for any physical stochastic reservoir. As we will see, by introducing stochasticity, the IPC of physical stochastic reservoirs is at most a polynomial in the number of output bits, even when considering all $2^n$ readout monomials (which form the conventional ``state space'' of the system). This makes it particularly important to be able to meaningfully select the ``best'' outputs, which requires some understanding of where the information is encoded. We will leave this problem to future work. In this section, we write only $p_k(u)$ to refer to the probability of bitstring $k$ at timestep $t$ when being driven by input $\boldsymbol{U}(t)$, however the reader should be aware that because the reservoir has memory, it would be more appropriate to write $p_k(u) = p_k(\boldsymbol{U}^{-a}(t))$ for some $a$ corresponding to the reservoir's effective memory, e.g. in \Cref{eq:fadingmem} $a=h_0$. We start with a lemma.
\begin{lemma}\label{lem:probrep}
{\rm\cite{2301.00042}}
The IPC in the probability representation is given as 
\begin{equation}
    \IPC  = \sum \frac{1}{1+\tilde{\sigma}_k^2} = Tr\left(\Delta(\overline{\langle \XX\rangle})^{-1}\overline{\langle\XX\rangle\langle\XX^T\rangle}\right),
\end{equation}
where $\Delta: \mathbb{R}^d \rightarrow \mathbb{R}^{(d,d)}$ maps a vector to the diagonal matrix with entries given by the vector, i.e.
\begin{equation}
    \Delta(\vec{a}) = \begin{pmatrix}
a_1 & 0 & \cdots & 0 \\
0 & a_2 & \cdots & 0 \\
\vdots  & \vdots  & \ddots & \vdots  \\
0 & 0 & \cdots & a_{d} 
\end{pmatrix}
\end{equation}
\end{lemma}
\begin{proof}
    \begin{equation}
    \begin{split}
     Tr\left((\II + \Qxit)^{-1}\right)=Tr((\DD^{-\frac12}\VV^T(\Delta(\overline{\langle \XX\rangle}) - \\
     \overline{\langle\XX\rangle\langle\XX^T\rangle})\VV\DD^{-\frac12} + \II)^{-1})\\
     =Tr\left((\overline{\langle\XX\rangle\langle\XX^T\rangle}^{-1}\Delta(\overline{\langle \XX\rangle}))^{-1}\right)\\
     = Tr\left(\Delta(\overline{\langle \XX\rangle})^{-1}\overline{\langle\XX\rangle\langle\XX^T\rangle}\right)
    \end{split}
\end{equation}
where we have expanded the variance with respect to the reservoir stochasticity, taking advantage of the fact the signals are Bernoulli random variables in the probability representation.
\end{proof}

\begin{theorem}\label{theorem:generateddesign}
The IPC of any physical stochastic reservoir is polynomial.
\end{theorem}
\begin{proof}
The right hand side of \Cref{lem:probrep}, with the notation in \Cref{def:probrep}, gives
\begin{equation}
   \IPC = \sum_k^{2^n} \frac{\int d\mu (u) p_k^2(u)}{\int d\mu (u) p_k(u)},
\end{equation}
for some measure $\mu$. For any value of $u$ we have that 
\begin{equation}
\sum^{2^n}_k  p_k(u) = 1.
\end{equation}
Because of this, the signals must generally be relatively small, and moreover when they are not small, they must decay rapidly. For each $p_k(u)$ we imagine that it has behavior in these decaying regions, which we will call ``tails'', proportional to some $1/g_k(u)$, i.e. $p_k(u)\sim 1/g_k(u)$, so that we have the condition 
\begin{equation}\label{eq:tailcondition}
\sum^{2^n}_k 1/g_k(u) = 1
\end{equation}
We might worry that they zero out in pathological ways, and so we might instead want to consider neighborhood's - Markov's inequality can achieve something similar to this discussion, but we avoid doing this for now.

For instance, a constant number can have constant tails, a polynomial number can have polynomial tails and any super polynomial number needs to have inverse super polynomial tails. Because the functions can only grow polynomially by \Cref{lemma:prob_changes}, we have that they have peaks that are $O(\poly(u)/g_k(u))$. We have so far written all functions as functions of $u$, but we note that, as previously discussed in \Cref{lemma:prob_changes}, the scale of the drive in any family of parameterized reservoirs will be related to $n$. Because we are integrating out $u$ below, we replace the functional dependence on $u$ with one on $n$. The IPC is thus bounded as given as 
\begin{equation}
   \IPC = \sum_k^{2^n} \frac{\int d\mu (u) p_k^2(u)}{\int d\mu (u) p_k(u)} \leq \sum_k^{2^n} \poly(n)/g_k(n) = \poly(n),
\end{equation}
where we have bounded each term based on the inequality $\int d\mu p_k^2(u) \leq \frac{\poly(n)}{g_k(n)}\int d\mu p_k(u)$, and used \Cref{eq:tailcondition}.
\end{proof}

\section{Connections to Learning Theory}\label{sec:connections}

In this section we discuss connections between the results proved in the previous section and modern ideas in statistical learning theory.
\subsection{A lower bound on the fat-shattering dimension}

In the context of machine learning and statistical learning theory, complexity measures are used to characterize the expressive power of hypothesis classes and bound generalization error. One such complexity measure is the \emph{fat-shattering dimension} \cite{kearns1994efficient}, a concept that extends the classical VC Dimension \cite{shalev2014understanding} to real-valued function classes, making it particularly suited for studying learning behavior of probabilistic classifiers and regression problems. We will start by introducing the fat-shattering dimension, and a theorem of Barlett et al. \cite{bartlett1994fat}. We will use this theorem to prove \Cref{theorem:nopoly} which states that the reservoir dynamics, i.e.
\begin{equation}\label{eq:fprime}
F' = \{p_{0...0}(t), \ldots, p_{1...1}(t)\},
\end{equation}
are not agnostically learnable (defined below in \Cref{def:agnostic}), and consequently have super polynomial fat-shattering dimension.

\begin{definition}
Let $\mathcal{X}$ be a domain of instances (i.e. an unlabeled data set) and let $\mathcal{H}$ be a class of real-valued functions mapping from $\mathcal{X}$ to $[0, 1]$, i.e., $h: \mathcal{X} \rightarrow [0, 1]$. Given a real value (the ``width'') $\gamma > 0$, the $\gamma$-fat-shattering dimension of $\mathcal{H}$, denoted by ${\rm fat}_\gamma(\mathcal{H})$, is defined as the largest natural number $d$ for which there exist $d$ instances in $\mathcal{X}$ and a set of thresholds $\{t_1, \ldots, t_d\} \subseteq [0, 1]$  such that for each subset $S \subseteq \{1, \ldots, d\}$, there is a function $h_S \in \mathcal{H}$ satisfying the following conditions:

\begin{itemize}
  \item For every $i \in S$, $h_S(x_i) \ge t_i + \gamma$.
  \item For every $i \notin S$, $h_S(x_i) \le t_i - \gamma$.
\end{itemize}
\end{definition}
The fat-shattering dimension captures the ability of a hypothesis class to have a substantial gap of at least $2\gamma$ between the values assigned by certain hypotheses to elements inside subset $S$ and elements outside subset $S$. Including the thresholds and the real-valued range of the functions makes the fat-shattering dimension valuable when studying learning behavior for probabilistic classifiers and regression problems. In this work, we consider a reservoir that has $2^n$ possible outputs. We imagine using these outputs to perform a classification task on the input signal by considering a linear combination of empirical estimates $\hat{p}_i(u)$ to perform binary classification on $u$. As a particularly illustrative example, consider ``switching signals'' (e.g \Cref{fig:exppoly}) which, upon receiving $u_i$ with $i\in S$, raises $p_{0...0}(u_i)$ above $0.5$ by at least $\gamma$ and with $i \notin S$ lowers $p_{0...0}(u_i)$ below $0.5$ by at least $\gamma$. Choosing $t_i=0.5$, such a reservoir has a fat-shattering dimension of at least $|S|$. Such a reservoir may not be implementable, however, given the specific dynamics available, or the details of the input signal. In particular, the example of ``switching signals'' (\Cref{fig:exppoly} is limited by the ability of physical system to drive large enough changes in the dynamics to produce these signals. Furthermore, we have considered the deterministic case here, where the functions $p_k(u)$ are treated as accessible real-valued functions. To accurately model the stochastic signals considered in this paper, we must consider the setting where these real-valued functions are instead corrupted by noise. In particular, we will assume they are the parameters of a Bernoulli distribution.

\begin{figure}
   \centering \includegraphics[scale=.4]{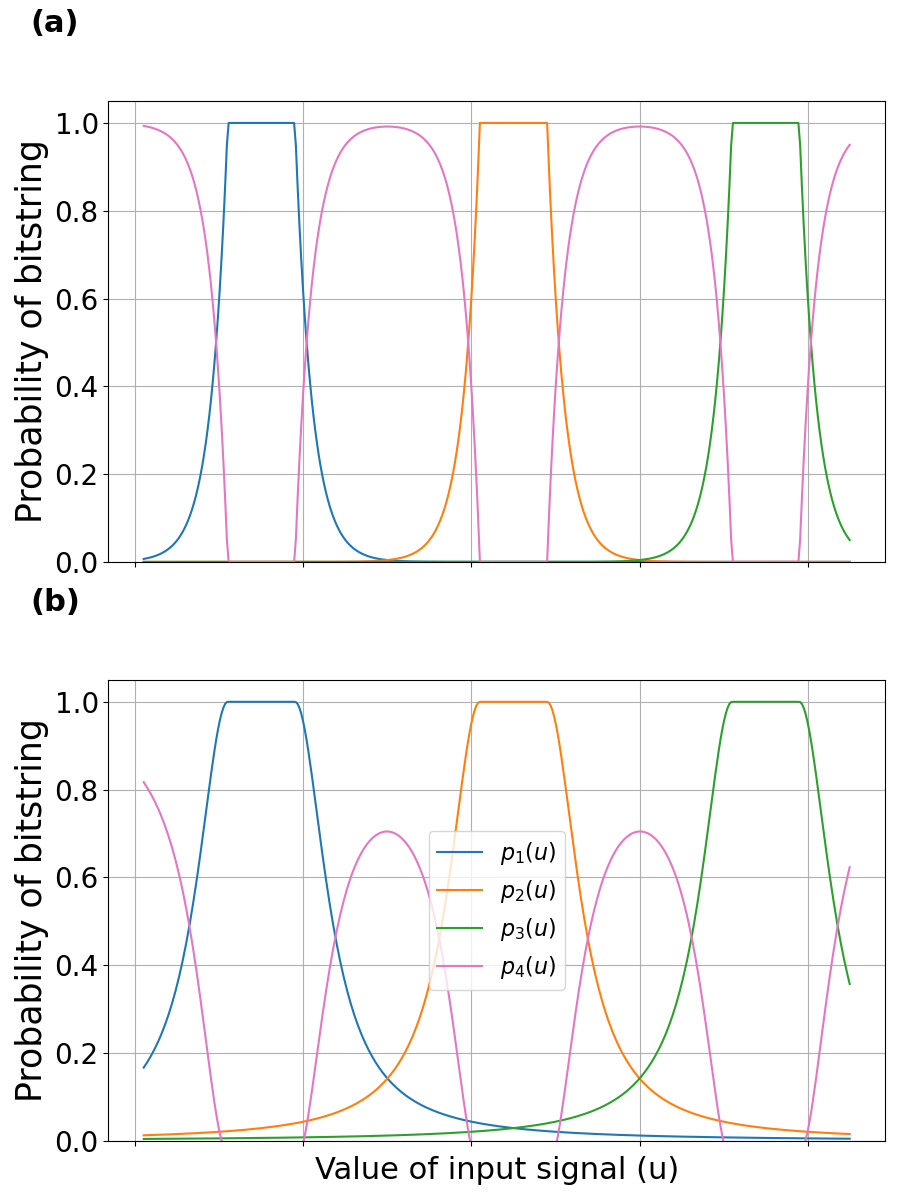}
    \caption{(a) ``Switching signals'', with exponential tails, proportional to $2^u$. With exponential changes in the probabilities, the reservoir is able to drive very rapid changes in the output signals, and thus able to fit four signals with trivial confusion probabilities near their peaks. In particular, observe that the red, low alpha signal is able to achieve a probability of nearly one between the first and second (blue and orange) and second and third (orange and green) signals. (b)``Switching signals'', with polynomial tails, proportional to $u^2$. The overlap where the signals cross is more substantial, and consequently fewer signals can be fit into the space. This results in significant rates of ``confusion'', e.g. the red signal has probabilities that are significantly smaller than 1 in the regions between pairs of the first and second (blue and orange), or second and third (orange and green) higher alpha signals.}
    \label{fig:exppoly}
\end{figure}

To this end, \cite{bartlett1994fat, kearns1994efficient, kearns1992toward,kearns1998efficient, haussler1992decision} consider the model of probabilistic computation. To start, we define probabilistic concepts and agnostic learning.
\begin{definition}{\rm\cite{kearns1994efficient}}
A probabilistic concept $f$ over a domain set $X$ is a mapping $f: X \rightarrow[0,1]$. For each $x \in X$, we interpret $f(x)$ as the probability that $x$ is a positive example of the probabilistic concept $f$. A learning algorithm in this framework is attempting to infer something about the underlying target probabilistic concept $f$ solely on the basis of labeled examples $(x, b)$, where $b \in\{0,1\}$ is a bit generated randomly according to the conditional probability $f(x)$, i.e., $b=1$ with probability $f(x)$. The value $f(x)$ may be viewed as a measure of the degree to which $x$ exemplifies some concept $f$. 
\end{definition}

To connect the capacity (\Cref{def:capacity}) to existing work in learning theory, we now define the error integral of a classifier.
\begin{definition}\label{def:error}
\begin{equation}
{\rm er}_P(h)=\int |h(x)-y| d P(x, y) = \overline{\langle \sqrt{(h(x)-y)^2}\rangle},
\end{equation}
denotes the error integral of a classifier $h$ trained on samples from a probability distribution $P(x, y)$, with the averaging notation being defined in \Cref{def:averages}.
\end{definition}

The following definition requires the learner to perform almost as well, with high probability, as the best hypothesis in some class $G$, referred to as a \textit{touchstone class}, for any particular learning task. The word agnostic in this setting is used because there is no assumption of an underlying function generating the training examples. We will consider a randomized learning algorithm which takes a sample of length $m$ and chooses sequences $z\in Z^m$ at random from $P^m_Z$, and gives it to a deterministic mapping $A$ as a parameter. Deterministic algorithms are a subset of these mappings where the $A$ ignores the random string. 

\begin{definition}{\rm \cite{bartlett1994fat}}\label{def:agnostic}
    Suppose $G$ is a class of $[0,1]$-valued functions defined on $X$, $P$ is a probability distribution on $X \times[0,1], 0<\varepsilon, \delta<1$, and $m \in \mathbb{N}$. A randomized learning algorithm $L$ is a pair $(A, P_Z)$, where $P_Z$ is a distribution on a set $Z$, and $A$ is a mapping from $\bigcup_m(X\times\mathbb{R})^m\times Z^m$ to $[0,1]^X$. For an algorithm $A$ and a distribution to be learned $D_Z$ on a set $Z$, we write that $L=\left(A, D_Z\right)$. We say that $L$ $(\varepsilon, \delta)$-learns in the agnostic sense with respect to $G$ from $m$ examples if, for all distributions $P$ on $X \times[0,1]$
$$
\begin{array}{r}
\left(P^m \times D_Z^m\right)\left\{(x, y, z) \in X^m \times[0,1]^m \times Z^m:\right. \\
\left.{\rm er}_P(A(x, y, z)) \geqslant \inf _{f \in G} {\rm er}_P(f)+\varepsilon\right\}<\delta,
\end{array}
$$
where ${\rm er}_P(\cdot)$ is the error integral introduced in \Cref{def:error}. The function class $G$ is agnostically learnable if there is a learning algorithm $L$ and a function $m_0:(0,1) \times(0,1) \rightarrow \mathbb{N}$ such that, for all $0<\varepsilon, \delta<1$, algorithm $L(\varepsilon, \delta)$-learns in the agnostic sense with respect to $G$ from $m_0(\varepsilon, \delta)$ examples. If, in addition, $m_0$ is bounded by a polynomial in $1 / \varepsilon$ and $1 / \delta$, we say that $G$ is small-sample agnostically learnable.

\end{definition}
 In our case, we will see that physical stochastic reservoir computers provide an example of a particular probability distribution for which learning the functions describing their dynamics is not small-sample agnostically learnable due to the degradation in capacity on those functions. We start with a corollary of \Cref{theorem:generateddesign} and a theorem by Bartlett et al., before proving our second theorem.

\begin{corollary}\label{cor:manybad}
For any physical stochastic reservoir there are $\Omega(g(n))$ learning tasks $f_i$ such that
\begin{equation}
   {\rm er}_P(f_i) \geq 1-O(\poly(n)/g(n)),
\end{equation} 
where $g(n) = \omega(\poly(n))$.
\end{corollary}
\begin{proof}
Because, for small errors ($|(h(x) - y)|\leq 1$), 
\begin{equation}
{\rm er}_P(h) \geq  \sqrt{\langle h^2 \rangle_T (1- C_T[h])},
\end{equation}
we see a small capacity also implies a large error. This will allow us to prove our corollary and connect with the statistical learning literature. Specifically, we see that
\begin{equation}
    {\rm er}_P(h) \geq \sqrt{\langle h^2\rangle_T} - \sqrt{\langle h^2\rangle_T}C_T[h]/2 + O(C_T[h]^2),
\end{equation}
for small capacity.

In \Cref{theorem:generateddesign} we argued that for any physical stochastic reservoir there are super-polynomially many functions in the touchstone class, $f_i\in F'$ (defined in \Cref{eq:fprime}), that have inverse superpolynomially poor capacity, i.e. $h$ such that
\begin{equation}
    C_T[h] = O(\poly(n)/g(n)),
\end{equation}
where $g(n) \in \omega(\poly(n))$.
We consider the learning problem with the touchstone class $F$
\begin{equation}
\begin{split}
F = \{f_i \mid & \text{$f_i$ is a linear combination of functions in } F' \\
               & \text{and } C_T(f_i) = O(\poly(n)/g(n)) \}.
\end{split}
\end{equation}

Note that in particular, this class includes the poor SNR eigentasks of the reservoir. The IPC, as we have seen, gives a measure of the SNR over each signal, and is normalized by definition. Replacing the capacity in our previous inequality with $\poly(n)/g(n)$, we have
\begin{equation}\label{eq:missingcap}
    {\rm er}_P(f) \geq \sqrt{\langle f_i^2\rangle_T}(1-O(\poly(n)/g(n))).
\end{equation} 
These functions form an orthonormal basis, and so we set the norm of the function to one, arriving at the desired inequality.
\end{proof}

This corollary relates the classification error of \Cref{def:error} of a reservoir performing an eigentask to its capacity, defined in \Cref{def:capacity}. Intuitively this is possible because the capacity is a measure of the SNR, and low SNR makes classification more difficult. We now give a theorem of Bartlett et al. that relates the fat-shattering dimension to small-sample agnostic learnability, and follow with our own theorem, showing that the fat-shattering dimension of the probabilistic concept class of reservoir functions has super-polynomial fat-shattering dimension. We start with a technical definition.

\begin{definition}{\rm \cite{bartlett1994fat, haussler1992decision}}
    Consider a $\sigma-$algebra $\mathscr{A}$ on $Z$. A class of functions $G$ is 
    $\mathbf{P H}$-permissible if it can be indexed by a set $T$ such that
    \begin{enumerate}
        \item $T$ is a Borel subspace of a compact metric space $\overline{T}$ and
        \item the function $f: Z\times T\rightarrow \mathbb{R}$ that indexes $G$ by $T$ is measurable with respect to the $\sigma-$algebra $\mathscr{A}\times\mathscr{B}(T)$, where $\mathscr{B}(T)$ is the $\sigma-$algebra of Borel sets on $T$.
    \end{enumerate}
    
We say a class $G$ of real-valued functions is permissible if the class $l_G : \{l_g\mid g\in G\}$, $l_g:(x,y)\to (y-g(x))^2$ is $\mathbf{P H}$-permissible. 

\end{definition}

\begin{theorem}{\rm \cite{bartlett1994fat}}\label{theorem:bart}
  Suppose $G$ is a permissible class of $[0,1]$ valued functions defined on $X$. Then $G$ is agnostically learnable if and only if its fat-shattering function is finite, and $G$ is small-sample agnostically learnable if and only if there is a polynomial $p$ such that ${\rm fat}_\gamma(G)<p(1 / \gamma)$ for all $\gamma>0$.
\end{theorem}

We will now demonstrate that due to the degradation in IPC, a learning algorithm cannot, in general, differentiate between the different learning tasks described by a reservoir's eigentasks without using an exponential number of observations. Hence we will demonstrate that there does not exist a learning algorithm that can agnostically learn the reservoir dynamics with $\poly(1/\delta, 1/\epsilon)$ samples. Specifically, we show that the assumption that the class of functions encoded by the dynamics of the reservoir is learnable in the presence of noise is not compatible with that class containing many orthogonal functions. If they are learned, they are learned despite the noise, and must all therefore be similar to a single learned function. But they cannot be too similar to the learned function, because then they would be similar to each other, and they are orthogonal.

\begin{theorem}\label{theorem:nopoly}
There does not exist a polynomial p such that ${\rm fat}_\gamma(F)<p(1/\gamma)$ for all $\gamma>0$ for the concept class of functions $F$ corresponding to reservoir dynamics of an infinite family of reservoirs.
\end{theorem}    

\begin{proof}
Since the errors are one-sided - the learner cannot perform better than the reservoir function at its own eigentask - 
the condition for agnostic learnability, is that for all $\epsilon > 0$ with probability $1-\delta > 0$, it is possible to take enough samples so that
\begin{equation}
    \sum^{|F|}_i {\rm er}_{P}(A(x, y, z)) \leq \sum^{|F|}_i{\rm er}_{P}(f_i) + \epsilon.
\end{equation}
\noindent Using \Cref{cor:manybad}, we can relate the error to the capacity, so that for the collection of functions $F$ with small capacity, this is equivalent to
\begin{equation}\label{eq:capbound}
 \sum^{|F|}_i {\rm er}_{P}(A(x, y, z))  \leq \sum^{|F|}_i (1 - c_i + \epsilon),
 \end{equation}
with high probability, where $c_i$ denote the terms that are $O(\poly(n)/g(n))$ in \Cref{eq:missingcap}. 

For small-sample agnostically learnability, we have the number of samples $m_0=\poly(1/\delta, 1/\epsilon)$. Because the identically zero function is among the functions that the learning algorithm must perform well on, i.e. $0\in F$, we can bound the the probability of success based on the probability that the learning algorithm falsely reports the identically zero function. For the functions we are considering, the capacity is low, and hence from our proof of \Cref{theorem:generateddesign} we can choose $n$ such that the probability of the learning algorithm sampling anything nonzero is small. In particular, for a probability $q$ that any particular sample is nonzero, the probability $p$ that the learner samples all zeros, consistent with the function being the zero function, is given as
\begin{equation}
p = (1-q)^{m_0} \approx m_0q \approx \poly(1/\delta, 1/\epsilon)\poly(n)/g(n),
 \end{equation}
for $g(n)\in \omega(\poly(n))$.

Making this approximation requires $m_0q \ll 1$, so that we choose $\poly(n, 1/\epsilon, 1/\delta)/g(n) \ll 1$. A sufficient condition, therefore, is choosing $n \approx \max{(1/\delta, 1/\epsilon)}$. In this case, the learner is unable to do better than guessing that the signal is small or zero, and incurring an error proportional to the size of the signal. Hence the condition for agnostic learnability becomes
\begin{equation}
\begin{split}
        \sum^{|F|}_i {\rm er}_{P}(A(x, y, z)) =  \sum^{|F|}_i(\int |f_i(x) - A(x,y,z)| d P(x, y) ) \\
        \geq  \sum^{|F|/2}_i\int |f_i(x) - f_{i+|F|/2}(x)| dP(x,y)
        \\
        \geq  \sum^{|F|/2}_i\int (f_i(x) - f_{i+|F|/2}(x))^2 dP(x, y) = |F|,
\end{split}
\end{equation}
where we have used the triangle inequality for the first inequality, $|f_i(x)-f_{i+|F|/2}(x)| \leq 1$ and $|f_i(x)-f_{i+|F|/2}(x)| \geq (f_i(x)-f_{i+|F|/2}(x))^2$ for the second inequality and orthogonality of the reservoir functions for the final equality. In particular, integrating over the input measure first, e.g.
\begin{equation}
\begin{split}
    \int d\mu(x) d\nu(y) f_1(x,y)f_2(x,y) = \\\int d\nu(y) \int d\mu(x) f_1(x,y)f_2(x,y) = 0,
\end{split}
\end{equation}
due to the orthogonality of the functions under $\mu$. Putting this together with \Cref{eq:capbound}, we have 

\begin{equation}
 |F| \leq \sum^{|F|}_i {\rm er}_{P}(A(x, y, z))  \leq \sum^{|F|}_i (1-c_i + \epsilon).
 \end{equation}

 so that this is violated if
 \begin{equation}\label{eq:req}
 \frac{1}{|F|}\sum^{|F|}_i c_i > \epsilon
 \end{equation}
 i.e. that the worst signals still provide some small utility on average. This requirement can be made arbitrarily weak since the condition for agnostic learnability is that this is true for all $\epsilon$, and hence what we have demonstrated is that there is an $(\epsilon, \delta, m_0)$ for which we can choose $n$ such that the reservoir dynamics are not small-sample agnostically learnable. An alternative interpretation is that any such learner would provide a witness that the eigentasks are not orthogonal under the reservoir dynamics, since this would require the eigentasks to be too similar to each other. Hence, we conclude from \Cref{theorem:bart} that there is no polynomial $p$ such that ${\rm fat}_\gamma(F)<p(1 / \gamma)$ for all $\gamma>0$ if \Cref{eq:req} holds.
\end{proof}

At this point we have taken a kind of limit, assuming that our reservoir is a member of an infinite family of reservoirs. For each $(\epsilon, \delta),$ we have required only that the number of samples scales polynomially with the reservoir size, and hence not only does the function class corresponding to the infinite family of reservoirs have a superpolynomial fat-shattering dimension, but also this result includes all settings where the functions that constitute this family are efficient (in the number of samples) to compute.  
\section{Discussion}

When we analyze a physical system---whether for thermodynamics, mechanics, or computation---we often imagine ``enlarging'' the system by adding degrees of freedom. In a reservoir-computing setting, one typically expects that adding more degrees of freedom increases the expressive or computational power of the reservoir. This intuition aligns with the idea that a larger system can encode more states, allowing richer dynamics or more intricate transformations of input signals.

In this paper, we have considered an analog encoding of our data, using the probabilities of the bitstrings to perform our computation. It is equally possible to encode a signal in the bitstrings themselves, and this is precisely what is done in digital computation, which is well known to be robust to these issues. For an analog system in a noisy or thermal environment, errors degrade performance, so there is effectively a signal-to-noise ratio that determines how well the reservoir can accomplish useful information processing. Our goal thus far has been to understand how fluctuations (i.e., noise) scale and how this scaling ties into energy dissipation. We will connect these ideas to the fluctuation--dissipation theorem (FDT) \cite{kubo1986brownian}, which provides a quantitative physical connection between noise and dissipative properties of physical variables.
\subsection{Fluctuation--Dissipation Perspective}
\label{subsec:FDT}
The classical fluctuation--dissipation theorem (FDT) \cite{kubo1986brownian} relates the spectral density to the temperature and dissipation,
\begin{equation}
  S_x(\omega) = \frac{2k_B T}{\omega} \Im[\chi(\omega)],
\end{equation}
where \(S_x(\omega)\) is the power spectral density of \(x(t)\). Hence we can imagine in our setting, the natural dynamics themselves, fluctuations, cause dissipation, and hence change the temperature. Such a temperature induces a thermal distribution on the outputs, and hence we see that there is natural noise that is proportional to the signal.



In the main text, we have modeled signals as probabilities in bitstrings.  For uncorrelated bits, variance is maximal at \(p=1/2\) (the ``middle'' of the Bernoulli range), whereas in the FDT scenarios the noise can be minimal at $x(t) = 0$, e.g. zero voltage.  However, one can construct correlated bits whose sum remains (nearly) pinned at a constant value, thus forcing the variance to vanish at certain points.  

Concretely, let \(X_1, X_2, \dots, X_n \in \{0,1\}\) be binary random variables each with 
\[
  \mathbb{P}(X_i = 1) \;=\; p.
\]
Define
\[
  X \;=\; \sum_{i=1}^n X_i.
\]
When the \(X_i\) are independent, \(X\sim \mathrm{Binomial}(n,p)\) with variance \(n\,p(1-p)\).  But one can correlate them to reduce fluctuations dramatically.  A well‐known example is to make \(S\) take only the integer values \(\lfloor np\rfloor\) or \(\lceil np\rceil\) with just the right probabilities, thereby achieving the minimal possible variance for given \(p\).  In particular, if \(p = 1/n\), one can make \(X = 1\) with probability 1, i.e.\ \(\mathrm{Var}(X)=0\).

Suppose we set
\[
  p \;=\; \frac{1}{n} + \epsilon
  \quad
  (\text{with }\epsilon\text{ small}),
\]
so that the expected sum is 
\(\mathbb{E}[X] = n\,p = 1 + n\,\epsilon.\)
We can realize this exactly using a two‐value random sum,
\[
  X \;=\;
  \begin{cases}
    1, & \text{with probability } 1 - n\epsilon,\\
    2, & \text{with probability } n\epsilon,
  \end{cases}
\]
and (implicitly) choose \emph{which} bits are 1 so that each has marginal \(p\).  Then $\mathrm{Var}(X)= n\epsilon $$(1 - n\epsilon)$. For \(\epsilon\to 0\), this vanishes linearly in \(\epsilon\).  

From the FDT viewpoint, the exact scaling of noise versus signal depends on how one identifies the degrees of freedom and the relevant ``signal'' parameter.  In an Ising‐like or Bernoulli‐bit system, probabilities \(p\) are dimensionless, and the minimal achievable variance at special points \(p = k/n\) necessarily vanishes, then increases as we perturb \(p\).  This perspective confirms that any physical system obeying a fluctuation--dissipation relation can, in principle, be approximated by a large collection of correlated binary degrees of freedom, reproducing the same qualitative phenomenon of ``noise growing with signal.'' While the variance in this example is linear in $\epsilon$, a simple majority voting scheme can be used to translate the linear variance of $X$ into a quadratic one, which may be more suitable depending on the specifics of the physical dynamics being considered.

\subsection{Integrated Circuits}
As a result of our construction in the beginning of the paper, we found the exponential degradation of the IPC even in the exponentially enlarged spaced consisting of all output bitstrings, resulting in the conclusion that the functions that describe any analog reservoir's dynamics have large fat-shattering dimension. Because the fat-shattering dimension lower bounds the generalization error, any analog learning scheme with a reservoir computer will face difficulty. Nevertheless, despite these difficulties we can attempt to identify areas where we might be able to produce useful analog reservoirs. To this end we will again consider integrated circuits as an example, motivated by our previous discussion. Generally integrated circuits exhibit Rentian scaling \cite{landman1971pin}, which is  that the observation that the typical number of connections in an integrated circuit $c$ is given as 
\begin{equation}
    c = tn^p,
\end{equation}
with $0 \leq p\leq 1$, and $t$ the average number of interconnects per circuit element $n$. In our model, we have considered $c=kn$ for some fixed $k$. In general, it is possible to consider families of reservoirs that allow for polynomial connectivity $k(n)$ rather than simply a constant connectivity $k$, however we still must constrain the connectivity to be $o(n)$ under physical considerations. Specifically, as the system grows with $n$, its total energy grows extensively in the number of circuit elements. If the energy density is allowed to grow with $n$ it will eventually collapse into a black hole. Consequently the space the reservoir occupies must grow with the number of circuit elements \cite{aaronson2005np}. Because communication over a distance requires time linear in that distance, and we have assumed polynomial constraints on time, we see that $d\sim \log(\poly(n))$ at most. In this derivation, we have made no assumption on the geometry of the system and the coupling can be arbitrarily long range, due to the generality of the theorem used in \Cref{app:physical}. As mentioned previously, the time to compile an arbitrary $o(n)-$bit operation is $2^{o(n)} = o(2^n)$. Consequently, any sublinear connectivity can be efficiently compiled in subexponential time to an fixed connectivity gate set, and we lose no generality in assuming a fixed $k$.

The degradation in the IPC that we have found is fundamentally due to the fact that we have encoded our signal into the same degree of freedom that is impacted by noise. To see this in a general physical example, we can model noise in an integrated circuit and consider Johnson-Nyquist noise on a voltage signal $V(t)$, due to an ideal resistor, characterized by a flat power spectral density,

\begin{equation}
S_V(\omega) = 4 k_B T R,
\end{equation}

where $k_B$ is Boltzmann’s constant, $T$ is the temperature, and $R$ is the resistance. Despite this noise being white, in realistic measurements we will observe the noise over a finite bandwidth $B$. Integrating the spectral density over the measurement bandwidth gives the in-band mean square noise power, $\langle P \rangle_{B}$ as
\begin{equation}
\langle P \rangle_{B} = \int_{-B/2}^{B/2} S_V(\omega) \, d\omega = 4 k_B T R B.
\end{equation}
The impact of the noise, as it enters the signal-to-noise ratio, is proportional to $\langle P \rangle _{B}$ and hence the temperature and the measurement bandwidth. The bandwidth $B$ represents the frequency range over which the system can resolve signals, which in turn is conjugate to the integration time. Thus, for low temperature the noise is small, and for fast electronics $B$ is large. Integrating for longer periods will decrease the error as $1/\sqrt{B}$ while increasing the RMS error as $\sqrt{B}$. In the case of Johnson noise, we see these effects perfectly balance, so that the only way to suppress the effects of noise are to cool the system. However, increasing the voltage increases the power dissipated, and so as before changing the signal is related to changing properties of the noise. For an efficient computation that consumes only a polynomial amount of power, the temperature will be at least $1/\poly(n)$, so that it is again impossible to rapidly cool the system to achieve large IPC - any such ``cooling and heating schedules'' would require ``switching signals'' as in \Cref{fig:exppoly}.

More generally, in this work we have argued that bits are the fundamental unit of computation, and consequently any analog signal should be constructed from the bitstring probabilities -- there is no other analog parameter to encode into. However, it is possible to motivate considering dits as might arise from a ``true'' analog element resulting in a relatively large effective number of bits, formed from an analog-to-digital converter. For $n$ dits with $d$ states, the available number of functions in the power set we considered is given by $(2^d)^n = 2^{dn}$. Consequently, making the signal ``effectively analog'' by choosing $d$ large can ameliorate the impact of these results. In particular choosing $d\sim 2^n/n$ could help overcome this degradation. Intuitively, each new state a dit can take on represents a new orthogonal direction that the analog signal can drive the ditstring probabilities. However, due to the previously discussed constraints on connectivity, this approach is limited due to physical scaling.

\subsection{Implications for Learning}

The empirical Rademacher complexity \( R_n(\mathcal{F}) \) measures how well functions in \( \mathcal{F} \) can correlate with random labels. A well-known result in empirical process theory states that if \( \text{fat}_\gamma(\mathcal{F}) > n \), then \( R_n(\mathcal{F}) \) must be at least a constant fraction of \( \gamma \) \cite{bartlett1999generalization}. Specifically, there exists a universal constant \( c > 0 \) such that:

\[
R_n(\mathcal{F}) \geq c \gamma \sqrt{\frac{F(\gamma)}{n}}.
\]

which follows from the contraction property of Rademacher complexity and chaining arguments. We've proven that there is no polynomial $ p(1/\gamma) $ such that:

\begin{equation}
    F(\gamma) < p(1/\gamma).
\end{equation}

The absence of such a polynomial bound implies that there exists a constant $ c > 0 $ such that for every polynomial $p$ and $\gamma > 0$,

\begin{equation}
    R_n(\mathcal{F}) \geq  c\gamma \sqrt{\frac{p(1/\gamma)}{n}}.
\end{equation}

The Rademacher complexity is typically used to upperbound the generalization error, hence this bound demonstrates an obstruction to bounding the generalization error of the reservoir dynamics through standard arguments, and specifically implies that there is a function $f$ describing an eigentask of any reservoir, such that for all $\gamma > 0$, polynomials $p$ and a constant $c > 0$,
\begin{equation}
\mathbb{E}[L(f)] - \frac{1}{n} \sum_{i=1}^{n} L(f(x_i)) \geq  c\gamma \sqrt{\frac{p(1/\gamma)}{n}},
\end{equation} 
for the loss function $L(f(x),y)=1$ if $(f(x)\neq y)$, and $0$ otherwise.

\section{Conclusion}
In this paper we have considered analog neural networks. We have argued that the only possible analog parameters in a digital model are the probabilities of the bitstrings and consequently we have considered stochastic, recurrent neural networks. A particularly simple example are reservoir computers, recurrent networks where the only trainable parameters are an output linear layer, which admit a simple formula for their IPC, which allows us to quantify the degradation to their performance due to the stochasticity. Specifically, by further considering the subset of these recurrent networks which we have called physical - those physically relevant in polynomial time and power - we have used a result from quantum complexity theory to demonstrate that this degradation can be meaningfully quantified as exponential. Further, borrowing results from classical learning theory, we then showed that this degradation in IPC implies that the functions that describe the reservoir's dynamics are difficult to learn, due to the natural noise causing confusion between the functions for any learner. 

Consequently, we were able to demonstrate they have a large fat-shattering dimension, and hence a large Rademacher complexity, formalizing the notion that these functions are susceptible to overfitting to noise. Because Rademacher complexity measures the ability of a set of functions to fit to random noise, it also provides a measure of overfitting to training data. Finally, this lets us lower bound the generalization error of learning with the reservoir. That is, we were able to show that any stochastic, physical reservoir computer necessarily has associated with it only a polynomially large amount of learning capabilities, even when it is allowed to access an exponentially large amount of post-processing. While it is well know that the information processing capacity is limited by the number of output signals this result strengthens the statement - even with an exponentially large number of output signals, any analog physical system subject to physicality constraints is limited to a polynomially small useful amount of learning. 

These results, in total, characterize a number of properties and behaviors of physical stochastic reservoir computers. Due to the abundance of stochasticity arising from sources such as friction, thermal and statistical uncertainty, and ultimately unavoidable shot noise due to quantum mechanics \cite{hu2023tackling, hu2023overcoming}, stochastic reservoir computers are a natural computational paradigm. Our first result, \Cref{theorem:generateddesign}, demonstrated that given a physical stochastic reservoir computer with $O(\poly(n))$ memory, the number of useful features (product signals) that can be produced from the $n$ output signals is $O(\poly(n))$.

By connecting ideas from learning with dynamical systems to concepts in statistical learning theory, we have also found that the fat-shattering dimension of the functions represented by reservoir dynamics is superpolynomial in the inverse of the fat-shattering width $\gamma$. Intuitively what we have shown is that, because reservoirs have a large number of low SNR eigentasks, and because no learning algorithm can be expected to do better than a reservoir at its eigentasks while being subjected to the same noise as the reservoir, the class of functions represented by the reservoir is itself challenging to learn. 

Surprisingly, this informs us about the growth of the fat-shattering dimension of the model at small scales ($\gamma \rightarrow 0$), whereas considerations from the dynamics in \Cref{theorem:nopoly} immediately rule out a collection of ``switching'' signals which instead seems to suggest a restriction on the growth of the function class at large scales ($1/\gamma\rightarrow 0$). In total, this work places large restrictions on the practical utility of large analog reservoir computers. While we have commented on the utility of small analog reservoir computers in the non-asymptotic regime, we leave the problem of optimizing the design of these devices subject to the constraints found in this paper open to future work. Specifically, this work highlights that while the exponentially large latent space of an analog, physical, stochastic reservoir computer is fundamentally inaccessible in the limit of long time, this still leaves at least a polynomially large space and potentially an exponentially large space in the limit of short times. Indeed, it is generally understood that analog computation is susceptible to confusion of noise with signal and the inability to separate the two - we have made this explicit in the setting of physical, recurrent, analog computation.

\section{Acknowledgments}
This material is based upon work supported by the U.S. Department of Energy, Office of Science, National Quantum Information Science Research Centers, Quantum Systems Accelerator (QSA). AMP thanks André Melo, Alexander Papageorge, Eric Peterson, C. Jess Riedel, Graeme Smith, Michael Walter and Reuben R. W. Wang for valuable feedback on the manuscript, and Nikolas Tezak for inspiring this line of work.

\bibliographystyle{IEEEtran}
\bibliography{main}

\appendix
\subsection{Physical Reservoirs and Proof of \Cref{lemma:prob_changes}}\label{app:physical}
The idea of physical or physically accessible circuits have been discussed in the literature (e.g. \cite{Shannon1949, poulin2011quantum}), however we are not aware of references that prove that general stochastic digital computers can only have polynomial changes in the probabilities. We thus include a proof of \Cref{lemma:prob_changes} here. \cite{coles2023thermodynamic} introduces the superoperator representation of Markovian dynamics. Markov chains, which describe the dynamics of our reservoir computer, have generators. Markov chains are equivalent descriptions of probabilistic circuits.
\begin{proof}
\cite{poulin2011quantum} establishes that arbitrary time-dependent quantum dynamics are no stronger (and clearly as strong as, by taking logarithms of the gates and applying them as Hamiltonians) the quantum circuit model. For this work, we consider reservoirs that are generated by stochastic circuit elements to avoid the details of the particular dynamics of the system. (This is simpler in the case of quantum dynamics due to the governing dynamics being given by the Schr\"odinger equation.) Physical circuits only have polynomial depth, and hence can only generate polynomial changes in values of s-modes \cite{coles2023thermodynamic} by the arguments in \cite{poulin2011quantum}.
\end{proof}

\end{document}